\documentclass[twoside]{article}

\usepackage[accepted]{aistats2026}
\usepackage[round]{natbib}

\usepackage{xcolor}
\usepackage{tabularx}
\usepackage{amsmath}
\usepackage{amsthm}
\usepackage{amsfonts}
\usepackage{mathtools}
\usepackage{bbm}
\usepackage{graphicx}
\usepackage{tabularx}
\usepackage{siunitx}
\usepackage{hyperref}
\usepackage{pgfplots}
\usepackage{subfigure}
\usepackage{booktabs}
\usepackage{thmtools}
\usepackage{thm-restate}
\usepackage{adjustbox}
\usepackage{titletoc}
\usepackage{pifont}
\DeclareSIUnit[quantity-product = ]\percent{\char`\%}
\newcommand{\capt}[1]{\mdseries{\emph{#1}}}

\hypersetup{
    colorlinks=true,
    citecolor=blue,
    linkcolor=blue,
    urlcolor=blue
}

\newtheorem{assumption}{Assumption}

\newtheorem{corollary}{Corollary}
\newtheorem{remark}{Remark}

\everymath=\expandafter{\the\everymath\displaystyle}

\newcommand\blfootnote[1]{%
  \begingroup
  \renewcommand\thefootnote{}\footnote{#1}%
  \addtocounter{footnote}{-1}%
  \endgroup
}

\begin{document}

%
\runningtitle{Computationally lightweight classifiers with frequentist bounds on predictions}

%

\twocolumn[

\aistatstitle{Computationally lightweight classifiers\\ with frequentist bounds on predictions}

\aistatsauthor{ Shreeram Murali \And Cristian R.\ Rojas \And Dominik Baumann }


\aistatsaddress{Cyber-physical Systems Group\\Aalto University \And Decision and Control Systems\\KTH Royal Institute of Technology \And Cyber-physical Systems Group\\Aalto University}
]


\begin{abstract}
    While both classical and neural network classifiers can achieve high accuracy, they fall short on offering uncertainty bounds on their predictions, making them unfit for safety-critical applications. Existing kernel-based classifiers that provide such bounds scale with $\mathcal O (n^{\sim3})$ in time, making them computationally intractable for large datasets. To address this, we propose a novel, computationally efficient classification algorithm based on the Nadaraya-Watson estimator, for whose estimates we derive frequentist uncertainty intervals. We evaluate our classifier on synthetically generated data and on electrocardiographic heartbeat signals from the MIT-BIH Arrhythmia database. We show that the method achieves competitive accuracy $>$\SI{96}{\percent} at $\mathcal O(n)$ and $\mathcal O(\log n)$ operations, while providing actionable uncertainty bounds. These bounds can, e.g., aid in flagging low-confidence predictions, making them suitable for real-time settings with resource constraints, such as diagnostic monitoring or implantable devices.
\end{abstract}

\section{INTRODUCTION}


Supervised classification is the basis for various categorization and instance-counting problems. Several classes of methods exist that combine high accuracy and computational efficiency. Approaches in classical machine learning, such as Logistic Regression or Support Vector Machines (SVMs) \citep{Shalev-Shwartz_Ben-David_2014, scholzkopf}, offer high accuracy with low computational cost. Nevertheless, deep learning methods are the de facto state-of-the-art for high accuracy; however, these models are often `black boxes', providing point predictions without an explainable confidence measure~\citep{LONGO2024102301}. To be adopted in data-abundant and safety-critical applications, classifiers must jointly address (i) accuracy, (ii) computational efficiency, and (iii) reliability. 

In classical machine learning, most methods lack reliable uncertainty quantification, preventing their adoption in high-stakes environments. Some classifiers, for example, SVMs, draw boundaries through data to delineate classes \citep{Shalev-Shwartz_Ben-David_2014, scholzkopf}. However, this is insufficient for uncertainty quantification; for instance, the classifier might be highly uncertain about a prediction near a boundary. Soft classifiers assign probabilities to each class \citep{Shalev-Shwartz_Ben-David_2014, rasmussen2008Gaussianprocessesmachine, baumann2024SafeReinforcementLearning}. Although the latter conveys more information about certainty, the probabilities cannot be interpreted as true confidence levels \textit{sua sponte} as they are often poorly calibrated (e.g., a model can be highly inaccurate while reporting high confidence). To address this, various uncertainty quantification techniques such as bootstrapping or deep ensembles have been employed, which train multiple models to estimate predictive variance \citep{hall1988TheoreticalComparisonBootstrap, pmlr-v48-gal16-anotherdropout}. While useful, these methods are often heuristic, computationally expensive, and typically do not provide formal guarantees on the prediction error.

For a more theoretically motivated approach to uncertainty quantification, Bayesian non-parametric methods, most notably Gaussian Process (GP) classification, are used \citep{rasmussen2008Gaussianprocessesmachine}. GPs provide a full posterior distribution over the class probabilities; however, exact GP inference is intractable in the classification setting, and even the approximate methods scale as $\mathcal O (n^3)$ with the number of data points. Furthermore, there are other disadvantages associated with a Bayesian framework beyond computational intractability; namely, a Bayesian model reflects its updated beliefs over a prior rather than a repeatable frequentist error interval. 

To address the latter, \cite{baumann2024SafeReinforcementLearning} derive frequentist uncertainty intervals using conditional kernel mean embeddings in a classification task. These embeddings are, however, computationally similar to GPs in that they are inefficient due to the need for matrix inversion; thus, the classifier from \cite{baumann2024SafeReinforcementLearning} comes at the same cost as GP regression. 


\paragraph{Contributions.}
To address the joint challenge of computational efficiency and rigorous frequentist uncertainty intervals, we propose a classifier that uses the Nadaraya-Watson (NW) estimator \citep{nadaraya-1964, watson1964SmoothRegressionAnalysis}. The NW estimator is a kernel density estimator whose typical use case is regression. In this paper, we reformulate the estimator as a classifier. We then derive frequentist bounds on its prediction errors and provide computationally lightweight implementations that further enhance its superior linear complexity. Our main contributions are: 
\begin{itemize}
    \item A non-parametric classification algorithm based on the Nadaraya-Watson estimator that scales linearly, $\mathcal{O}(n)$, with the size of the training set.
    \item The derivation of frequentist uncertainty bounds on the estimated class probabilities, valid for both overlapping data distributions and well-separated distributions.
    \item Computationally efficient variations on the naive implementation to improve from linear to sublinear and logarithmic computational complexity.
    \item Validation on both synthetic and real-world medical data, demonstrating that our method achieves competitive accuracy while providing actionable uncertainty bounds at a fraction of the computational cost of competing methods.
\end{itemize}

\section{PROBLEM SETTING}
\label{sec:problem}

We define the underlying probability space as $(\mathbf \Omega, \mathcal F, \mathcal P)$ with random variables $Y : \mathbf \Omega \to \mathcal Y$ and labels $C : \mathbf \Omega \to \mathcal C$ that take values in $\mathcal Y$ and $\mathcal C$ respectively. Here, $\mathbf \Omega$ is the sample-space, $\mathcal F$ is a sigma-algebra on $\mathbf \Omega$, and $\mathcal P$ is a measure that assigns probabilities to events in $\mathcal F$.

In this work, we estimate $p_c(y) \coloneqq \mathbb P(C = c \mid Y = y)$, the probability of observations $y \in \mathcal Y \subseteq \mathbb R ^ d$ belonging to class $c \in C \in \mathbb N$. We also derive high probability uncertainty bounds on the estimate of $p_c$. That is, for each class $c$ and a user-defined probability of at least $1 - \delta$, we show that the error between the true probability $p_c$ and its estimate $\hat{p}_c$ is bounded as a function of $y$, $\delta$, and sample size $n$, as
\begin{align}
  \label{eq:bound}
  |p_c(y) - \hat{p}_c(y)| \leq \epsilon_c(y, \delta, n).
\end{align}

Since the nature of the true probability function $p_c(y)$ is not known, we introduce assumptions about the underlying data distribution from which observations $y$ have been sampled. These assumptions cover two cases: one where the underlying data distribution is overlapping, and the other where it is separable. Here, we note that only one of the following two assumptions about the distribution must hold. 

\paragraph{Overlapping distributions.} 

For overlapping distributions, we assume the underlying probability function is Lipschitz continuous. This assumption captures scenarios wherein different classes may share similar characteristics yet remain distinguishable through smooth transitions in the probability space. 

\begin{assumption}
    \label{ass:lipschitz}
    The true probability function $p_c(y)$ is Lipschitz continuous with a known Lipschitz constant $L < \infty$. That is, for each $c \in \mathcal C$, and any two samples $y$ and $y'$,
    \begin{align}
        \label{eq:lipschitz}
        |p_c(y) - p_c(y')| \leq L \| y - y' \|. 
    \end{align}
\end{assumption}

In this paper, we use $\| \cdot \|$ to denote the $L^2$ norm. Thus, $\| y - y' \|$ represents the Euclidean distance between the two samples $y$ and $y'$. 

\begin{remark}
\label{remark:lipschitz}
Assuming knowledge of the Lipschitz constant \textit{a priori} is common in the control and safe learning literature \citep{Magureanu:2014, brunke2022safe}. It can generally be approximated from data, to which end the majority of existing approaches use the classical estimator from \cite{Strongin1973}: 
\begin{align}
    \label{eqn:strongin}
    \hat{L} \coloneqq r \max_{i \neq j} \frac{|f(x_i) - f(x_j)|}{\| x_i - x_j \|}, 
\end{align}
where $r \in \mathbb R$ is a multiplicative factor, $(x_i, f(x_i))$ is a data sample, and $f$ is the unknown function to be estimated. 

In our experiments, we use an approach similar to~\eqref{eqn:strongin} (see Appendix~\ref{sec:estimating-L}) to approximate a Lipschitz constant. Nevertheless, various approaches have been built on \cite{Strongin1973}; for instance, \cite{wood1996EstimationLipschitzconstant} fit an approximate distribution to the Lipschitz estimate in the one-dimensional case, and \cite{Sergeyev:1995} uses this approach to the multi-dimensional case by using space-filling curves to solve a global optimization problem. Further, \cite{Novara:2013} and \cite{Calliess:2020} extend this estimator to handle bounded observational noise, while \cite{huangSampleComplexityLipschitz} provide finite-sample guarantees on the estimate with stronger assumptions on the target function. 


Yet, the common theme in these approaches is to invoke a regularity assumption on the underlying unknown function. This is what we do through Assumption~\ref{ass:lipschitz}. 
\end{remark}


\paragraph{Separable distributions.} 

For separable distributions, we assume that samples from different classes are well-separated in the feature space, with a known minimum distance between them. 

\begin{assumption}
    \label{ass:separable}
    Samples in $\mathbf \Omega$ are separable with a known margin $\gamma$. That is, for any two samples $(y, c)$ and $(y', c')$ drawn from $\mathcal D$, where $c \neq c'$,
    \begin{align}
        \label{eq:separable}
        \| y - y' \| \geq \gamma, 
    \end{align}
    with probability 1. 
\end{assumption}

In other words, the distribution $\mathcal D$ has a margin $\gamma$ if the distance between points with differing labels is at least $\gamma$. 


\paragraph{Nature of measurements.} Next, we require an assumption about the nature of sampling from our dataset. 

\begin{assumption}
    \label{ass:iid-assumption}
    Samples $(y_i, c_i)$ are independently drawn from the same distribution $\mathcal D$. 
\end{assumption}

This independent and identical distribution (i.i.d.) assumption is often invoked in the literature to provide theoretical guarantees \citep{rao1996PAClearningfunctions}. In our case, we introduce this assumption to derive data-dependent bounds on the sampling error. 


\section{CLASSIFIERS}

In this section, we first formulate the NW estimator as a classifier in Section~\ref{sec:nwc}, derive bounds on the estimate in Section~\ref{sec:bounds-deriving}, and propose improvements to an already efficient naive implementation in Section~\ref{sec:comp-eff-improvements}.

\subsection{Nadaraya-Watson classifier}
\label{sec:nwc}

From the standard form of the Nadaraya-Watson estimator~\citep{nadaraya-1964, watson1964SmoothRegressionAnalysis}, we replace the real-valued observation with the indicator function $\mathbbm 1 _{c_i} \!\coloneqq\! \mathbbm 1 \{ c_i\! =\! c \}$, a one-hot vectorized representation of the class label $c_i$. In doing so, we reformulate the Nadaraya-Watson estimator to estimate the probabilities $\mathbf p_c( y )$ as 
\begin{align}
    \label{eq:nwc}
    \hat{\mathbf p }_c( y) = \frac{1}{\kappa_n( y)} \sum_{i=1}^{n} K_\lambda ( y, y_i) \mathbbm{1}_{c_i},
\end{align}
where 
\begin{align}
    \label{eq:nwc-kernel-definitions}
    \kappa_n( y ) & \coloneqq \sum_{i=1}^{n} K_\lambda ( y, y_i), \nonumber \\ 
    K_\lambda( y, y_i ) & \coloneqq \frac{1}{c_k} K \left( \frac{\| y - y_i \| }{\lambda} \right). 
\end{align}

Here, $K(\cdot)$ is the kernel function that conveys the degree of similarity between a query sample $y$ and training sample $y_i$, $\lambda$ is the user-defined bandwidth parameter, $c_k$ is a normalization constant, and $\mathbbm 1 _{c_i} \!\coloneqq\! \mathbbm 1 \{ c_i\! =\! c \}$ is the one-hot vectorized representation of the class label $c_i$. We use $\mathbf p_c(y)$ to denote the vector of assigned class probabilities across all classes, and $p_c(y)$ to describe the scalar probability for a single class $c$. 

\begin{assumption}
    \label{ass:kernel}
    The kernel $K : \mathbb R ^d \to \mathbb R$ is non-negative and bounded such that for any $v\in \mathbb R^d$ and some $c_k < \infty$,  $0 \leq K(v) \leq c_k$, and $K(v) = 0$ for all $\| v \| > 1$. 
\end{assumption}

Since the kernel function is user-defined, we can satisfy this assumption by explicitly selecting a suitable kernel or formulating one in line with this assumption. For instance, many popular kernels, such as the boxcar or the Epanechnikov kernel, already meet this assumption; for those that do not, such as the Gaussian (Radial Basis Function; RBF) kernel, the outputs for $\| v \| > 1$ can be explicitly truncated. 

Furthermore, in Appendix~\ref{sec:extended_kernel_proof} we relax Assumption~\ref{ass:kernel} to include non-bounded kernels in exchange for a new assumption on bounded inputs to extend the validity of our bounds to kernels with infinite support.

The computational time of the NW estimator scales linearly with a naive implementation, but as we show in Section~\ref{sec:comp-eff-improvements}, this can be improved to sublinear computational complexity with some preprocessing.

\subsection{Deriving bounds on the estimates}
\label{sec:bounds-deriving}

We provide theoretical worst-case bounds on the estimate produced in \eqref{eq:nwc} by splitting the error into two sources: the uncertainty that is inherent in not obtaining samples of the true probability function $p_c(y)$ but only discrete labels, and the sampling error that arises when estimating a function from a finite number of samples. We refer to the former as the classifier's \textit{bias} and the latter as its \textit{sampling error}. 

In this section, we derive bounds on the error between the true probability function $p_c(y )$ and its estimate $\hat{p}_c (y )$ as depicted in~\eqref{eq:bound}. We start by deriving the estimator's bias corresponding to the two cases in Assumptions \ref{ass:lipschitz} and~\ref{ass:separable}. 

To this end, we introduce a virtual estimate $\bar{p}_c (y)$, a quantity that could hypothetically be determined if we had true probability measurements instead of discrete labels. With this, we split \eqref{eq:bound} into two terms, one corresponding to the estimator's bias and the other to its statistical error in sampling: 
\begin{align}
    \label{eq:bound-split}
    |p_c(y) - \hat{p}_c(y)| \leq \underbrace{|p_c(y) - \bar{p}_c(y)|}_{\text{bias}} + \underbrace{|\bar{p}_c(y) - \hat{p}_c(y)|}_{\text{sampling error}}. 
\end{align}
We prove that these two terms are individually bounded. Then, by the triangle inequality, the term on the left side of \eqref{eq:bound-split} would also be bounded. 
    
\subsubsection{Bias}

We first analyse the bias term under the two scenarios mentioned in Section \ref{sec:problem}: when the underlying probability function is Lipschitz continuous, and when the data distributions are separable.

\begin{restatable}{lemma}{lemmaone}
    \label{lemma:bias_lipschitz}
    Under Assumptions \ref{ass:lipschitz} and \ref{ass:kernel}, we have, for all $n \geq 0$ and $y \in \mathcal Y$, 
    \begin{align}
        \label{eq:overlapping}
        |p_c(y) - \bar{p}_c(y)| \leq L \lambda,
    \end{align}
    where $L$ is the known Lipschitz constant from \eqref{eq:lipschitz} and $\lambda$ is the user-defined kernel bandwidth from~\eqref{eq:nwc-kernel-definitions}. 
\end{restatable}
The proofs of Lemma~\ref{lemma:bias_lipschitz} and further technical results are collected in Appendix~\ref{sec:proofs} and~\ref{sec:bounds_sampling_error}. This lemma shows that for overlapping distributions, the bias of our estimator is bounded by the product of the Lipschitz constant and the kernel bandwidth. Intuitively, this means that the bias increases linearly with both the rate of change of the probability function and the size of the local neighbourhood we consider for estimation.

\begin{restatable}{lemma}{lemmatwo}
    \label{lemma:bias_margin}
    Under Assumptions \ref{ass:separable} and~\ref{ass:kernel}, we have for almost every $y \in \mathcal Y$ 
    \begin{align}
        \label{eq:margin}
        |p_c(y) - \bar{p}_c(y)| \leq \frac{\lambda}{\gamma}, 
    \end{align}
    where $\lambda$ is the user-defined kernel bandwidth parameter as depicted in \eqref{eq:nwc-kernel-definitions} and $\gamma$, in accordance with Assumption \ref{ass:separable}, is the known margin of the distribution $\mathcal D$. 
\end{restatable}

For separable distributions, \eqref{eq:margin} demonstrates that the bias is bounded by the ratio of the kernel bandwidth to the margin between classes. This suggests that larger margins between classes allow for larger kernel bandwidths while maintaining the same bias bound.

Additionally, this result can be extended to positive-definite kernels with infinite support. To do this, we trade the compactness of the kernel in Assumption~\ref{ass:kernel} for an assumption on bounded input space. 
\begin{assumption}
    \label{ass:bounded_input}
    There exists a finite diameter $\Phi$ such that for any sample $y_i$ and input $y \in \mathcal{Y}$, 
    \begin{align}
        \| y - y_i \| \le \Phi.  
    \end{align}
\end{assumption}

This assumption lets us bound the bias of the classifier with an extra term that represents the kernel-weighted sum of the number of samples outside a chosen bandwidth $\lambda^\ast$. 
\begin{restatable}{lemma}{lemmatwopointfive}
    \label{lemma:infinite_support_bias}
    Under Assumptions~\ref{ass:lipschitz}, \ref{ass:separable}, \ref{ass:kernel} with the change that $K(v) \geq 0$ for all $\| v \| > 1$, and~\ref{ass:bounded_input} we have for almost every $y \in \mathcal{Y}$ 
    \begin{align}
        |p_c(y) - \bar{p}_c(y)| \leq \beta \lambda^\ast + \beta \Phi \varepsilon_t, 
    \end{align}
    where
    \begin{align*}
        \varepsilon_t \coloneqq \sum_{i \in \mathcal{I}_{\mathrm{far}}} \frac{K_\lambda (y, y_i)}{\kappa_n(y)} \| y - y_i \|
    \end{align*}
    is a term that represents the weighted sum corresponding to the samples in the tail of the kernel's span (indices $\mathcal{I}_{\mathrm{far}}$, where $\| y - y_i \| > \lambda^\ast$), $\Phi$ is the input space diameter from Assumption~\ref{ass:bounded_input}, and $\beta = L$ or $1/\gamma$ depending on whether we assume an overlapping (see Assumption~\ref{ass:lipschitz}) or a separable (see Assumption~\ref{ass:separable}) distribution on $\mathcal D$.
\end{restatable}

We prove this Lemma in Appendix Section~\ref{sec:extended_kernel_proof}. However, since this is a more conservative bound, we use the bounds from Lemmas~\ref{lemma:bias_lipschitz} and~\ref{lemma:bias_margin} in our experiments (Section~\ref{sec:experiments}). Since there is greater freedom in choosing a kernel than in choosing data, we consider Lemma~\ref{lemma:infinite_support_bias} as a theoretically valid extension applicable to a few limited datasets. 

\subsubsection{Sampling error}
\label{sec:sampling_error}

Having bounded the bias term, we now analyse the sampling error, which captures the uncertainty due to finite sample estimation.

\begin{restatable}{lemma}{lemmathree}
    \label{lemma:sampling_error}
Under Assumptions~\ref{ass:iid-assumption} and~\ref{ass:kernel}, we have, for all $n \geq 0$, with probability at least $1-\delta$, 
\begin{align}
\label{eq:lemma5}
    | \bar{p}_c(y) - \hat{p}_c(y)| \leq 2 \sigma \frac{\alpha_n(y, \delta)}{\kappa_n(y)},
\end{align}
where
\begin{align}
\label{eq:alphadef}
    \alpha_n(y, n) = 
        \begin{cases}
            \sqrt{\kappa_n(y) \log(\delta^{-1}\sqrt{1 + \kappa_n(y)})},  \\ \hfill  \text{if } \kappa_n(y) > 1, \\
            \sqrt{\log \left( \sqrt 2/\delta \right)},  \\ \hfill \text{if } 0 < \kappa_n(y) \leq 1.
        \end{cases} 
\end{align} 
\end{restatable}

This lemma provides a probabilistic bound on the sampling error that depends on the local density of samples through $\kappa_n(y)$. The bound becomes tighter in regions with more samples (larger $\kappa_n(y)$), reflecting the intuition that predictions are more reliable in data-dense regions. The two cases in the definition of $\alpha_n$ handle differently the scenarios of sparse and dense sampling.

\subsubsection{Combined bounds}

Correspondingly, we now formulate overall bounds on the prediction error by gathering the bounds from~\eqref{eq:overlapping},~\eqref{eq:margin}, and~\eqref{eq:alphadef}. 

\begin{corollary}
    Under the same assumptions as Lemma~\ref{lemma:sampling_error} and either Lemma~\ref{lemma:bias_lipschitz} or Lemma~\ref{lemma:bias_margin}, we have, with a probability of at least $1 - \delta$,
    \begin{align}
        \label{eq:overall-bounds}
        |p_c(y) - \hat{p}_c(y)| \leq \beta \lambda + 2 \sigma \frac{\alpha_n(y, \delta)}{\kappa_n(y)}, 
    \end{align}
    where $\alpha_n(y, \delta)$ is the data-dependent term from \eqref{eq:alphadef} and $\beta = L$ or $1/\gamma$ based on the nature of the data's underlying probability distribution. 
\end{corollary}

\begin{proof}
    This follows from applying the triangle inequality and collecting the bias term from either Lemma~\ref{lemma:bias_lipschitz} or~\ref{lemma:bias_margin}, and the sampling error from Lemma~\ref{lemma:sampling_error}. 
\end{proof}

\subsection{Computational efficiency improvements}
\label{sec:comp-eff-improvements}

The naive implementation of the proposed classifier scales with $\mathcal{O}(n)$, which is a significant improvement over the closest competing alternative by \cite{baumann2024SafeReinforcementLearning}, which scales with $\mathcal{O}(n^3)$. Nevertheless, we can further improve its efficiency by taking inspiration from $k$-nearest neighbour-based methods \citep{nnyaba2024EnhancingElectrocardiographyData} and implement a \textit{localized} variant. To this end, we employ k-d trees to facilitate the lookup of k-nearest neighbours. Building a k-d tree takes $\mathcal O (n \log n)$ time, and querying scales with $\mathcal O (k + \log n)$. Thus, for small $k$, we have approximately logarithmic scaling behaviour, while for increasing $k$, the outputs and the performance of the \textit{localized} classifier approach that of the \textit{regular} implementation.  


As an alternative, we implement a hash table variant of the proposed classifier in \eqref{eq:nwc} that performs lookup with $\mathcal O (\log n)$ complexity. We refer to this implementation in successive text as the \textit{dyadic} classifier. The construction of the hash table scales with $\mathcal O (n)$. However, this approach suffers from two crucial limitations. First, it does not allow for the efficient computation of bounds. This is because $\kappa_n(y)$ is a query-dependent value: it is the sum of weights calculated from the distances between a new query sample and the training samples. Range trees and hash tables are data structures built solely on training data; they cannot perform distance-based calculations relative to a new query sample $y$. Second, this approach suffers from the curse of dimensionality. For a user-defined resolution parameter $m$, the number of dyadic cells scales with increasing feature dimension $d$ in the order of $2^{m^d}$. 

\begin{table}[ht]
    \vspace{-1.5em}
    \centering
    \caption{Time complexity of different NWC implementations compared with CMEs and Logistic Regression.}
    \vspace{0.5em}
        \begin{tabular}{lcc}
            \toprule
            & Preprocessing & Querying \\
            \midrule
            \textbf{\textit{Regular}} & -- & $\mathcal{O}(n)$ \\
            \textbf{\textit{Dyadic}} & $\mathcal{O}(n)$ & $\mathcal{O}(\log n)$ \\
            \textbf{\textit{Localized}} & $\mathcal{O}(n \log n)$ & $\mathcal{O}(k + \log n)$ \\
            \textbf{CME} & -- & $\mathcal O (n^3)$\\
            \textbf{LR} & $\mathcal{O}(nd)$ & $\mathcal{O}(d)$\\
            \bottomrule
            \vspace{-1.5em}
        \end{tabular}
    \label{tab:nwc_time_complexity}
\end{table}

We compare the theoretical time complexities of the proposed classifier, its variants, and baseline methods Logistic Regression \citep{Shalev-Shwartz_Ben-David_2014} and CMEs \citep{baumann2024SafeReinforcementLearning} in Table~\ref{tab:nwc_time_complexity}.

\begin{remark}
    \cite{rao1996PAClearningfunctions} also present a variant of the NW estimator, improving from the naive $\mathcal O (n)$ to a faster $\mathcal O ((\log n)^d)$ implementation using range trees. With this data structure, the method enables the computation of the number of points within an arbitrary hyper-rectangle. However, this is misaligned with the objectives of classification; the estimation of any point within a cube is a pre-calculated aggregate of the training data that fell into said cube. Thus, the prediction problem is not one of searching in an arbitrary range; it is one of a direct lookup operation. These operations are better handled using hash tables. 
\end{remark}

\begin{remark}
    Since $K(v)\! \coloneqq \!0$ for $v \!\geq\! 1$ (Assumption~\ref{ass:kernel}), we can interpret the regular classifier as a local estimator whose locality is defined by the bandwidth parameter. The localized implementation is thus an alternative framing of the regular classifier, wherein we look up the $k$-nearest neighbours and compute their weights rather than iterating through the entire dataset to compute kernel-weights for samples within the bandwidth.
\end{remark}

\section{EXPERIMENTS}
\label{sec:experiments}

We implement and evaluate our proposed classifier and its variants against baseline methods in two settings: synthetically generated data to validate our theoretical assumptions, and a real-world electrocardiogram (ECG) dataset to demonstrate its practical utility in a safety-critical application. Additionally, we compare the methods on the MNIST dataset in Appendix~\ref{sec:mnist}. The code is available in the supplementary material.\blfootnote{ \href{https://github.com/shreeram-murali/AISTATS-2026}{github.com/shreeram-murali/AISTATS-2026}}

\paragraph{Synthetic data.} 
\label{sec:synthtic_data_para}


 We create two datasets, one Lipschitz continuous, and another separated by a margin. For the former, we define the true underlying probability function $p_c(y)$ using a logistic function with a direct relationship to $L$ (see Appendix~\ref{sec:synthetic_data_creation}). We sample points from the probability function such that it matches Assumption \ref{ass:lipschitz}. For the latter, we generate class centres spaced by a margin greater than $\gamma$ and sample points from a small cluster surrounding the centre. Figure \ref{fig:synthetic} shows the two datasets used in our results.

\begin{figure}[htpb]
    \centering
    \begin{minipage}[t]{0.9\textwidth}
    \includegraphics{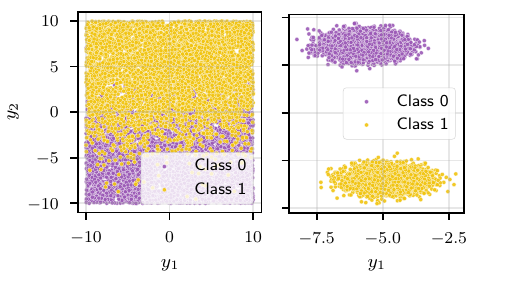}
    \label{fig:overlapping_synthetic}
    \end{minipage}
    \caption{The two synthetic datasets used for evaluating the classifier. \capt{One dataset with overlapping classes, adhering to Assumption~\ref{ass:lipschitz} (left), and one with classes separated by a known margin, in accordance with Assumption~\ref{ass:separable} (right).}}
    \label{fig:synthetic}
\end{figure}

\paragraph{Electrocardiographic data.}

We use the MIT-BIH Arrhythmia database, a widely-used benchmark in cardiology research \citep{moody2001impactMITBIHArrhythmia}.  The dataset contains approximately \SI{110000} annotated heartbeats from 48 half-hour recordings. Each heartbeat sample is represented by a 187-dimensional vector, which we truncated to contain the first 100 features to retain the most informative part of the QRS complex. We followed data preprocessing steps identical to those in prior work (e.g., \cite{nnyaba2024EnhancingElectrocardiographyData}). The dataset was split into \SI{87554} training and \SI{21892} testing samples.  Heartbeats are categorized into five classes\footnote{defined by the Association for the Advancement of Medical Instrumentation (AAMI) EC57 standard} as: Normal (N), Supraventricular ectopic (S), Ventricular ectopic (V), Fusion (F), and Unclassifiable (Q).  The dataset exhibits a significant class imbalance, which reflects real-world clinical scenarios.

\subsection{Results}
\label{sec:results}

\begin{figure*}[h!]
    \centering
    \begin{minipage}[t]{0.32\textwidth}
        \centering
        \resizebox{\textwidth}{!}{\input{figures/total_time_vs_training_size.pgf}}
    \end{minipage}
    \begin{minipage}[t]{0.32\textwidth}
        \centering
        \resizebox{\textwidth}{!}{\input{figures/runtime_vs_training_size.pgf}}
    \end{minipage}
    \begin{minipage}[t]{0.32\textwidth}
        \centering
        \resizebox{\textwidth}{!}{\input{figures/margin_fit_time_vs_training_size.pgf}}
    \end{minipage}
    \centering
    \begin{minipage}[t]{0.32\textwidth}
        \centering
        \resizebox{\textwidth}{!}{\input{figures/bounds_vs_training_size.pgf}}
    \end{minipage}
    \begin{minipage}[t]{0.33\textwidth}
        \centering
        \resizebox{\textwidth}{!}{\input{figures/accuracy_vs_training_size.pgf}}
    \end{minipage}
    \begin{minipage}[t]{0.33\textwidth}
        \centering
        \resizebox{\textwidth}{!}{\input{figures/margin_accuracy_training_size.pgf}}
    \end{minipage}
        \begin{minipage}[t]{\textwidth}
        \centering
        \vspace{-1.8em}
        \includegraphics[height=1.5cm]{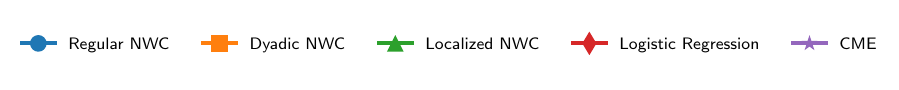}
    \end{minipage}
    \caption{Performance of the proposed classifiers on the synthetic datasets compared to baselines with varying sample sizes. \capt{We plot total runtime, prediction time, fit time (top row); average bounds $\bar{\epsilon}_c$ for $\delta=0.05$ for all classes, accuracy on the overlapping and separable dataset (bottom row). We observe that our algorithm is significantly more sample-efficient than the CME-based classifier, while achieving high accuracy with minimal uncertainty.}}
    \label{fig:synthetic_results}
    \vspace{-0.8em}
\end{figure*}

For the synthetic overlapping dataset, we set $L = 0.15$, and for the separable dataset, we set the margin as an equivalent $\gamma = 6.67$. For the \textit{regular} and \textit{localized} NWC variants, we use the Epanechnikov kernel, defined as 
\begin{align}
    \label{eq:ep_kernel}
    K(v) = 
    \begin{cases}
        1 - v^2, & \text{if } \|v\|  \leq 1, \\ 
        0, & \text{otherwise,}
    \end{cases}
\end{align}
with bandwidth $\lambda \!=\! 0.2$. In Appendix Section~\ref{sec:ablations} we provide an ablation study to find the best kernel and $\lambda$ that maximizes our classifier's accuracy and bounds. The results in Figure~\ref{fig:synthetic_results} depict the performance of the proposed and baseline methods. Here, we see that CMEs' computational costs were prohibitively expensive; we run the classifier only up to a maximum of \SI{10000} samples. The top row of figures shows the superior computational efficiency of the \textit{regular}, \textit{localized}, and \textit{dyadic} variants over CMEs. From the second row of Figure~\ref{fig:synthetic_results}, we see the uncertainty intervals tightening significantly with increasing $n$. Additionally, for accuracy, we observe that the computational efficiency of the proposed classifiers does not come at the cost of requiring larger training sets. 

Next, we summarize results from electrocardiographic data. On the MIT-BIH dataset, we set the Lipschitz constant to $L=0.05$ and the kernel bandwidth to $\lambda=0.75$ (see Appendix Section~\ref{sec:estimating-L}) with the Epanechnikov kernel. Figure~\ref{fig:ecg_results} shows the performance of the \textit{regular} and \textit{localized} classifier on this dataset, with illustrative examples of predictions, bounds, and the assigned class probabilities. When trained on the full dataset of over \SI{87554} samples, our \textit{regular} classifier achieves an accuracy of \SI{96.2}{\percent}, while the \textit{localized} one has an even higher accuracy of \SI{97.8}{\percent}. This comparison, along with precision-recall scores and runtimes, is shown in Table~\ref{tab:ecg_detailed_results} in a detailed evaluation of the classifiers with varying training set sizes. While CMEs provide tighter bounds at smaller set sizes, their cubic complexity makes them intractable for the full dataset. The proposed classifier, even with a naive implementation, was over 500 times faster to evaluate on the entire \SI{87554} sample strong dataset than CMEs were on just \SI{10000} samples. 


\begin{figure*}[h!]
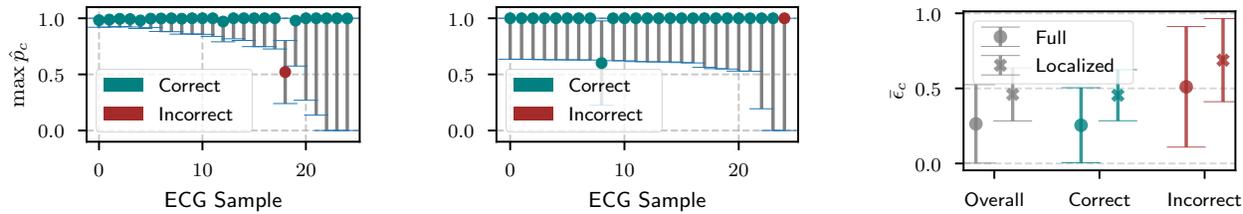
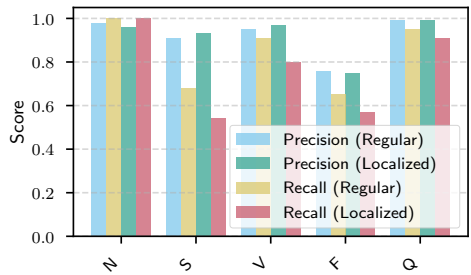
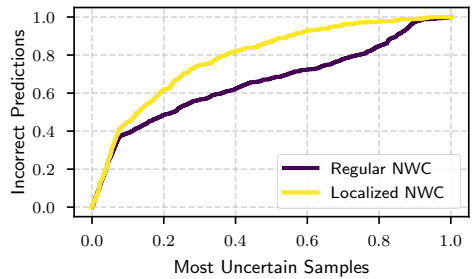

    \centering
    \begin{minipage}[t]{0.32\textwidth}
        \centering
        \subfigure[\textit{Regular} NWC: predicted probabilities and \SI{95}{\percent} bounds ($\delta=0.05$).]{%
            \resizebox{\textwidth}{!}{\input{figures/regular_smaller.pgf}}
            \label{fig:ecg_1}
        }
    \end{minipage}
    \hfill
    \begin{minipage}[t]{0.32\textwidth}
        \centering
        \subfigure[\textit{Localized} NWC: predicted probabilities and \SI{95}{\percent} bounds ($\delta=0.05$).]{%
            \resizebox{\textwidth}{!}{\input{figures/localized_predictions_smaller.pgf}}
            \label{fig:ecg_2}
        }
    \end{minipage}
    \hfill
    \begin{minipage}[t]{0.32\textwidth}
        \centering
        \subfigure[Mean bound width vs.\ predicted probability.]{%
            \resizebox{\textwidth}{!}{\input{figures/full_vs_localized_mean_bound_13.pgf}}
            \label{fig:ecg_3}
        }
    \end{minipage}

    \vspace{1em} 

    \begin{minipage}[t]{0.49\textwidth}
        \centering
        \subfigure[Precision-recall metrics for all classes and both proposed classifier variants: \textit{regular} and \textit{localized} NWC.]{%
            \resizebox{0.8\textwidth}{!}{\input{figures/precision_recall.pgf}}
            \label{fig:ecg_prec_rec}
        }
    \end{minipage}
    \hfill
    \begin{minipage}[t]{0.49\textwidth}
        \centering
        \subfigure[The cumulative recall curves. \capt{The curves show the proportion of total errors identified when samples are ranked by descending uncertainty.}]{%
            \resizebox{0.8\textwidth}{!}{\input{figures/crc_both.pgf}}
            \label{fig:crc}
        }
    \end{minipage}

    \caption{Mean uncertainty intervals, precision-recall metrics for the proposed classifiers, and waveforms. \capt{Our classifier shows higher uncertainty in misclassified labels, as well as high precision and recall scores.}}
    \label{fig:ecg_results}
\end{figure*}

\begin{table*}[h!]
    \centering
    \caption{Acccuracy and runtime comparison of the classifiers on the ECG dataset. \capt{We observe that both our approaches outperform logistic regression and the CME-based classifier in terms of accuracy, while being computationally orders of magnitude cheaper than the latter.}}
    \vspace{0.5em}
    \begin{tabular}{llccccc}
        \toprule
        \textbf{Size} & \textbf{Classifier} & \textbf{Accuracy} & \textbf{Precision} & \textbf{Recall} & \textbf{Precompute (s)} & \textbf{Query Time (s)} \\
        \midrule
        1\,000 & \textbf{\textit{Regular NWC}} & 86.4 & 0.881 & 0.864 & -- & 0.140 \\
        & \textbf{\textit{Localized NWC}} & 91.0 & \textbf{0.903} & 0.910 & \textbf{0.0008} & 0.085 \\
        & \textbf{Logistic Regression} & 88.8 & 0.859 & 0.888 & 0.0128 & \textbf{0.0002} \\
        & \textbf{CME} & \textbf{92.0} & 0.846 & \textbf{92.0} & -- & 8.883 \\
        \addlinespace
        10\,000 & \textbf{\textit{Regular NWC}} & 93.4 & 0.938 & 0.934 & -- & 1.573 \\
        & \textbf{\textit{Localized NWC}} & \textbf{96.4} & \textbf{0.963} & \textbf{0.964} & \textbf{0.0145} & 0.395 \\
        & \textbf{Logistic Regression} & 91.6 & 0.894 & 0.916 & 0.161 & \textbf{0.0002} \\
        & \textbf{CME} & 92.0 & 0.846 & 0.920 & -- & 8618.85 \\
        \addlinespace
        60\,000 & \textbf{\textit{Regular NWC}} & 95.6 & 0.958 & 0.956 & -- & 11.607 \\
        & \textbf{\textit{Localized NWC}} & \textbf{97.6} & \textbf{0.977} & \textbf{0.976} & \textbf{0.188} & 1.729 \\
        & \textbf{Logistic Regression} & 91.0 & 0.908 & 0.910 & 5.929 & \textbf{0.0003} \\
        & \textbf{CME} & -- & -- & -- & -- & -- \\
        \addlinespace
        87\,554 & \textbf{\textit{Regular NWC}} & 96.2 & 0.965 & 0.962 & -- & 16.905 \\
        & \textbf{\textit{Localized NWC}} & \textbf{97.8} & \textbf{0.979} & \textbf{0.978} & \textbf{0.318} & 1.691 \\
        & \textbf{Logistic Regression} & 91.0 & 0.903 & 0.910 & 6.527 & \textbf{0.0002} \\
        & \textbf{CME} & -- & -- & -- & -- & -- \\
        \bottomrule
    \end{tabular}
    \label{tab:ecg_detailed_results}
\end{table*}

Crucially, the uncertainty bounds provided by the proposed classifiers are actionable. As seen in Figure~\ref{fig:ecg_3}, incorrect predictions are typically associated with lower confidence (a lower probability estimate $\hat{p}_c$) or higher uncertainty. Correspondingly, in Figure~\ref{fig:crc} we show the cumulative recall curve; for instance, we observe that flagging the top \SI{10}{\percent} of the most uncertain predictions captures roughly \SI{40}{\percent} of the classifiers' incorrect predictions. This allows for a system where uncertain predictions can be automatically flagged for manual review by a clinician, a critical feature for deployment in healthcare. The \textit{localized} NWC emerges as the most accurate and computationally efficient method, while the \textit{regular} NWC provides stronger theoretical bounds. Both variants significantly outperform the baselines in achieving a practical balance of accuracy, efficiency, and reliability.

Each experiment in this section addresses a specific case: the synthetic data allows for the creation of datasets with known Lipschitz-continuity and margins, and the ECG data highlights the NW classifiers' core competence in balancing accuracy, reliability, and computational efficiency. The proposed classifier successfully bridges the gap between efficient but non-guaranteed methods like Logistic Regression, and methods with formal guarantees like CMEs, which are computationally prohibitive \citep{baumann2024SafeReinforcementLearning}. Our proposed variants---\textit{localized} and \textit{dyadic} implementations---further offer a tunable trade-off: they provide higher computational efficiency at the cost of more conservative bounds. In regulated industries like healthcare, where diagnosticians often manually review data, it is desirable to have methods that are accurate and also express uncertainty. Here, the role of bounds becomes clearer: measurements can thus be flagged for manual review based on the strength of their bounds.  


\section{RELATED WORK}
\label{sec:discussion-related-work}

In this section, we contextualize our contributions and position our work against the existing literature in uncertainty quantification, non-parametric methods, and previous work on ECG heartbeat classification. 

\paragraph{Frequentist uncertainty quantification.}
Many contemporary uncertainty quantification techniques, particularly in deep learning, rely on empirical methods like Monte-Carlo Dropout or deep ensembles \citep{pmlr-v48-gal16-anotherdropout, aiaa-bilal}. These approaches are empirical and lack the formal, high-probability error guarantees necessary for high-stakes decision-making. On the other hand, while Bayesian methods offer a more principled method for uncertainty quantification, they are non-frequentist, and are computationally and analytically intractable \citep{rasmussen2008Gaussianprocessesmachine, villacampa-calvo2021MulticlassGaussianProcess}. The most direct frequentist predecessor to our work, the CME-based classifier by \cite{baumann2024SafeReinforcementLearning}, provides the desired formal guarantees but suffers from the same prohibitive $\mathcal{O}(n^3)$ complexity due to its reliance on matrix inversion. Our work directly addresses this limitation.

\vspace{-0.5em}

\paragraph{Non-parametric kernel regression.}
Our choice of the Nadaraya-Watson estimator is motivated by its history as a non-parametric regressor \citep{nadaraya-1964, watson1964SmoothRegressionAnalysis}. It has typically found its usage in statistical applications \citep{schuster1979ContributionsTheoryNonparametric, nadaraya1989NonparametricEstimationProbability, blsprao}. Moreover, the estimator has also been applied to system identification problems in control theory \citep{schuster1979ContributionsTheoryNonparametric, juditsky1995Nonlinearblackboxmodels, ljung2006ASPECTSNONLINEARSYSTEM, mzyk2020Wienersystemidentification}. Nevertheless, the most relevant usage of the NW estimator to this paper's contributions has been in the safe learning literature as a computationally efficient alternative to GPs \citep{baumann2024computationallylightweightsafe}. It has been successfully used to provide safety and optimality guarantees in regression settings for reinforcement learning and control \citep{kowalczyk2024KernelBasedLearningGuarantees, baumann2025SafetyOptimalityLearningBased}. Nevertheless, we are not aware of any previous work that reformulates the NW estimator as a multi-class classifier and derives frequentist uncertainty bounds on its class probability estimates. In doing so, we inherit the computational efficiency of the estimator and couple it with the formal guarantees typically found in the domain of computationally expensive methods. 

\vspace{-0.5em}

\paragraph{Approaches to arrhythmia detection.}
The implications of our work are most evident in the context of our ECG experiments. Classification in the context of arrhythmia detection spans a large class of problems: some methods aim to classify strips of rhythms \citep{heden_detection_1996}, comprising multiple heartbeats, while others aim to classify individual beats themselves \citep{kachuee2018ECGHeartbeatClassification}. Many classical and deep learning approaches have been successful in yielding highly accurate predictions on the MIT-BIH dataset, ranging from \SI{95.9}{\percent} to \SI{99.87}{\percent} \citep{kumari2022ClassificationECGbeatsa, zhou2024MultimodalECGheartbeat, abdalla2020Deepconvolutionalneural, gao2019EffectiveLSTMRecurrent, jha2020Cardiacarrhythmiaclassification}, but these approaches often disregard the uncertainty associated with predictions (see Appendix~\ref{sec:baselines}). Some recent studies have attempted to quantify uncertainty through Monte Carlo dropout simulations \citep{zhang2022DeepBayesianNeural} and variational encoders \citep{barandas2024Evaluationuncertaintyquantification}. The most relevant contribution by \cite{nnyaba2024EnhancingElectrocardiographyData} utilized a $k$-nearest neighbours based GP classification method \citep{muyskens2021MuyGPsScalableGaussian}. While \cite{nnyaba2024EnhancingElectrocardiographyData} demonstrate strong results on the ECG dataset, their bounds are less rigorous, non-frequentist, and the classifier scales cubically with the number of neighbours $k$ considered. 

\section{CONCLUSIONS}


We have presented a novel classification algorithm that reformulates the Nadaraya-Watson estimator for multi-class classification tasks. Our work jointly addresses three challenges in modern machine learning: computational efficiency, theoretical guarantees, and practical applicability in safety-critical domains. The classifier achieves linear time complexity, $\mathcal{O}(n)$, a significant improvement over the cubic scaling of existing methods that provide formal guarantees. This efficiency does not come at the cost of reliability; we derive rigorous frequentist bounds on prediction errors, providing guarantees for both overlapping and separable data distributions.

We complement the theoretical contributions with practical implementations that further enhance computational efficiency. Our localized variant, leveraging k-d trees, achieves sublinear time complexity while maintaining high accuracy, reaching \SI{97.8}{\percent} on real-world ECG data, though with more conservative bounds. The \textit{dyadic} implementation offers even faster lookups at $\mathcal{O}(\log n)$, however, without the ability to specify kernels or compute bounds. These variants provide practitioners with flexible options to balance computational resources against uncertainty quantification needs.

\paragraph{Future work.} Despite promising results, our approach has limitations suggesting directions for future research. A primary challenge is estimating the Lipschitz constant from data, which can be non-trivial \citep{wood1996EstimationLipschitzconstant, tokmak2025safeexplorationreproducingkernel}. Another challenge would be to overcome the expressivity of Euclidean distances in higher dimensions. Future work could also explore special cases---for instance, ordinal and sequential classification---to tighten theoretical bounds, enhancing the classifier's utility.

\section*{Acknowledgements}

This research was partially supported by the Research Council of Finland flagship programme: the Finnish Center for Artificial Intelligence (FCAI), the Tandem Industry Academia Seed funding from the Finnish Research Impact Foundation, the Swedish Research Council under contract number 2023-05170, the Wallenberg AI, Autonomous Systems and Software Program (WASP), and the Swedish Civil Defence and Resilience Agency (Project MAD-VAMCHS). We also acknowledge the computational resources provided by the Aalto Science--IT project.


\bibliographystyle{apalike}
\bibliography{ref}


\newpage

\section*{Checklist}



\begin{enumerate}

  \item For all models and algorithms presented, check if you include:
  \begin{enumerate}
    \item A clear description of the mathematical setting, assumptions, algorithm, and/or model. [Yes]
    \item An analysis of the properties and complexity (time, space, sample size) of any algorithm. [Yes]
    \item (Optional) Anonymized source code, with specification of all dependencies, including external libraries. [Yes]
  \end{enumerate}

  \item For any theoretical claim, check if you include:
  \begin{enumerate}
    \item Statements of the full set of assumptions of all theoretical results. [Yes]
    \item Complete proofs of all theoretical results. [Yes]
    \item Clear explanations of any assumptions. [Yes]     
  \end{enumerate}

  \item For all figures and tables that present empirical results, check if you include:
  \begin{enumerate}
    \item The code, data, and instructions needed to reproduce the main experimental results (either in the supplemental material or as a URL). [Yes]
    \item All the training details (e.g., data splits, hyperparameters, how they were chosen). [Yes]
    \item A clear definition of the specific measure or statistics and error bars (e.g., with respect to the random seed after running experiments multiple times). [Yes]
    \item A description of the computing infrastructure used. (e.g., type of GPUs, internal cluster, or cloud provider). [Not Applicable]
  \end{enumerate}

  \item If you are using existing assets (e.g., code, data, models) or curating/releasing new assets, check if you include:
  \begin{enumerate}
    \item Citations of the creator If your work uses existing assets. [Yes]
    \item The license information of the assets, if applicable. [Not Applicable]
    \item New assets either in the supplemental material or as a URL, if applicable. [Yes]
    \item Information about consent from data providers/curators. [Not Applicable]
    \item Discussion of sensible content if applicable, e.g., personally identifiable information or offensive content. [Not Applicable]
  \end{enumerate}

  \item If you used crowdsourcing or conducted research with human subjects, check if you include:
  \begin{enumerate}
    \item The full text of instructions given to participants and screenshots. [Not Applicable]
    \item Descriptions of potential participant risks, with links to Institutional Review Board (IRB) approvals if applicable. [Not Applicable]
    \item The estimated hourly wage paid to participants and the total amount spent on participant compensation. [Not Applicable]
  \end{enumerate}

\end{enumerate}

\clearpage
\appendix
\startcontents[appendices]
\thispagestyle{empty}

\onecolumn

\aistatstitle{Computationally lightweight classifiers\\ with frequentist bounds on predictions: Appendix}

\printcontents[appendices]{}{1}{\section*{Contents}}

\clearpage

\section{BOUNDS ON THE BIAS}
\label{sec:proofs}

In this section, we restate and prove Lemmas \ref{lemma:bias_lipschitz}, \ref{lemma:bias_margin}, and \ref{lemma:infinite_support_bias} from Section~\ref{sec:bounds-deriving}.  

\subsection{Proof of Lemma \ref{lemma:bias_lipschitz}}
\label{sec:proof1}

\lemmaone*
\begin{proof}
\label{proof:lipschitz}
For notational convenience, we denote the kernel weight of the estimator at each iteration, ${K_{\lambda}(y, y_i)}/{\kappa_n(y)}$, from \eqref{eq:nwc} by $\theta_i$. Notably, these weights sum to $1$. The virtual estimate $\bar{p}_c(y)$ in the bias term can be represented as the weighted sum of the true probability at each observation:
\begin{align}
\label{eq:weights-nwc}
    |p_c(y) - \bar{p}_c(y)| = \left| \, p_c(y) - \sum_{t=1}^n \theta_i p_c(y_i) \, \right|.
\end{align}
Furthermore, since the weights sum to $1$, the term $p_c(y)$ can be introduced into a summation term without altering its value:
\begin{align}
\label{eq:weights-constant}
    p_c(y) = p_c(y) \sum_{i=1}^n \theta_i = \sum_{i=1}^n \theta_i p_c(y).
\end{align}
Equation \eqref{eq:weights-constant} can be used to rewrite and simplify \eqref{eq:weights-nwc} as
\begin{align}
    |p_c(y) - \bar{p}_c(y)| &= \left| \, \sum_{i=1}^n \theta_i p_c(y) - \sum_{i=1}^n \theta_i p_c(y_i) \, \right|  \nonumber \\
    &= \left| \, \sum_{i=1}^n \theta_i (p_c(y) - p_c(y_i)) \, \right|.  \nonumber
\end{align}
By the Cauchy-Schwarz inequality, 
\begin{align}
\label{eq:weights-triangle}
    \left| \, \sum_{i=1}^n \theta_i (p_c(y) - p_c(y_i)) \, \right| \leq  \sum_{i=1}^n \theta_i \left| p_c(y) - p_c(y_i) \right|.
\end{align}

However, under Assumption \ref{ass:lipschitz}, the term $|p_c(y) - p_c(y_i)|$ is bounded by a known, finite Lipschitz constant $L$ as 
\begin{align}
\label{eq:bias1-lipschitz}
    \left| p_c(y) - p_c(y_i) \right| \leq L\|y - y_i\|. 
\end{align}
Furthermore, due to Assumption \ref{ass:kernel}, if $K(y, y_i) \geq 0$, then $\theta_i \geq 0$, and $\|y - y_i\| \leq \lambda$. Therefore, by combining \eqref{eq:weights-triangle} and \eqref{eq:bias1-lipschitz}, we obtain
\begin{align}
\label{eq:l-lambda}
    |p_c(y) - \bar{p}_c(y)| & \leq \sum_{i=1}^n \theta_i \left| p_c(y) - p_c(y_i) \right| \nonumber \\
    & \leq \sum_{i=1}^n \theta_i L \|y - y_i\| \leq L \lambda.
\end{align}
\end{proof}

\subsection{Proof of Lemma \ref{lemma:bias_margin}}

\lemmatwo*
\begin{proof}
    \label{proof:bias_margin}
Due to Assumption \ref{ass:separable}, $p_c(y)$ is equivalent to a function $f: \mathcal Y \to \mathcal [0,1]$ with Lipschitz constant $1/\gamma$ for almost all $y \in \mathcal Y$. This would imply that $\mathbbm 1_{c_i} = f(y_i)$ with probability $1$. For example, if we define a set $\mathcal R$ 
\begin{align}
    \mathcal R \coloneqq \{ y \in \mathcal Y : p_c(y) = 0 \},  \nonumber
\end{align}
we can choose $f$ as
\begin{align}
    f(y) = \min \left\{  1, \frac{1}{\gamma} \inf_{y' \text{ in } \mathcal R}  \|y - y'\|\right\}. \nonumber
\end{align}
From here, the proof proceeds identically as that of Lemma~\ref{lemma:bias_lipschitz} from Appendix Section \ref{sec:proof1}. The terms $p_c(y)$ from \eqref{eq:bias1-lipschitz} could be replaced with $f(y)$, following which the result from \eqref{eq:l-lambda} would remain valid for almost all $y \in \mathcal Y$. 
\end{proof}

\subsection{Extension to positive definite kernels with infinite support}
\label{sec:extended_kernel_proof}

Assumption~\ref{ass:kernel} limits kernel choice to those with support only in $[-1, 1]$. While this can easily be accomplished for any kernel by truncating its outputs to zero outside the interval $[-1, 1]$, this aspect of Assumption~\ref{ass:kernel} can be relaxed by assuming bounded inputs. In this case, we can extend our results to positive definite kernels with infinite support. Doing so would conservatively expand Lemmas~\ref{lemma:bias_lipschitz} and~\ref{lemma:bias_margin} to accommodate positive definite kernels with infinite support.

\lemmatwopointfive*
\begin{proof}
    We introduce a cut-off radius $r^\ast$ such that and use the Cauchy-Schwarz inequality (similar to~\eqref{eq:l-lambda}) to split $|p_c(y) - \bar{p}_c(y)|$ into two terms, with indices in $\mathcal I _{\mathrm{near}}$ for instances where $\| y - y_i \| \le \lambda^\ast$, and $\mathcal I _{\mathrm{far}}$ corresponding to indices where $\| y - y_i \| > \lambda^\ast$, where $r ^ \ast = \lambda ^ \ast / \lambda$. We can write that split as 
    \begin{align}
        \label{eq:weights-triangle-infinite-support}
        |p_c(y) - \bar{p}_c (y)| \leq \beta \sum_{i \in \mathcal{I}_{\mathrm{near}}} \theta_i \| y - y_i \|  + \beta \sum_{i \in \mathcal{I}_{\mathrm{far}}} \theta_i \| y - y_i \|.
    \end{align}

    Since the kernel weights sum to 1, samples in the near set $\mathcal I_{\mathrm{near}}$ are bounded by $\lambda^\ast$ as
    \begin{align}
        \beta \sum_{i \in \mathcal I _{\mathrm{near}}} \theta_i \| y - y_i \| \le \beta  \lambda^\ast \sum_{i \in \mathcal I_{\mathrm{near}} } \theta_i \le \beta \lambda ^\ast. 
    \end{align}

    For samples in the $\mathcal I _{\mathrm{far}}$ set, the distances and weights can be large, but they are capped by Assumption~\ref{ass:bounded_input}:
    \begin{align}
        \label{eq:capped-weights}
        \sum_{i \in \mathcal{I}_{\mathrm{far}}} \theta_i \| y - y_i \| \le \Phi  \sum_{i \in \mathcal{I}_{\mathrm{far}}} \theta_i.
    \end{align}
    We set $\varepsilon_t \coloneqq \sum_{i \in \mathcal{I}_{\mathrm{far}}} \theta_i$, to obtain the bounds 
    \begin{align}
        \label{eq:extended-bias}
        |p_c(y) - \bar{p}_c(y)| \leq \beta \lambda^\ast + \beta \Phi \varepsilon_t, 
    \end{align}
    where $\beta = L$ or $1/\gamma$. 
\end{proof}

The choice of kernel heavily influences the practical utility of the bounds from~\eqref{eq:extended-bias}. For a standard RBF kernel, setting $r^\ast \coloneqq 3$ captures \SI{99.7}{\percent} of the kernel's mass. Yet, in large datasets, the number of points in the $\mathcal{I}_{\mathrm{far}}$ set vastly outnumber those in the $\mathcal{I}_{\mathrm{near}}$ set. This would cause the tail term $\beta \Phi \varepsilon_t$ to blow up, resulting in the bounds being unreasonably conservative. 

\section{BOUNDS ON THE SAMPLING ERROR}
\label{sec:bounds_sampling_error}

In this section, we restate and prove Lemma~\ref{lemma:sampling_error} from Section~\ref{sec:sampling_error}.

\subsection{Concentration inequality for self-normalized sums with sub-Gaussian noise}

Here, we restate a result from \cite{baumann2024computationallylightweightsafe, baumann2025SafetyOptimalityLearningBased} which we use to prove Lemma \ref{lemma:sampling_error}. This lemma introduces a concentration inequality for self-normalized sums with sub-Gaussian noise \cite[Thm~3]{abbasiyadkori2011onlinesquaresestimationselfnormalized}. 

\begin{restatable}[\cite{baumann2024computationallylightweightsafe, baumann2025SafetyOptimalityLearningBased}]{lemma}{lemmafour}
\label{lemma:precursor}
Let $\{v_t : t \in \mathbb N \}$ be a bounded stochastic process and $\{ \omega_t : t \in \mathbb N \}$ be an i.i.d.\ sub-Gaussian stochastic process, i.e., there exists a $\sigma > 0$ such that, for any $\rho \in \mathbb R$, and any $t \in \mathbb N$, 
\begin{align}
    \mathbb E [\exp(\rho \omega_t)] \leq \exp \left( \frac{\rho^2 \sigma^2}{2} \right). \nonumber
\end{align}
Further, let $S_n \coloneqq \sum_{t = 1}^n v_t \omega_t $ and $V_n \coloneqq \sum_{t=1}^n v_t^2$. Then for any $n \in \mathbb N$ and $0 < \delta < 1$, with probability at least $1- \delta$, 
\begin{align}
    |S_n|  \leq \sqrt{2 \sigma^2 \log(\delta^{-1} \sqrt{1 + V_n})(1 + V_n)}. 
\end{align}
\end{restatable}

Correspondingly, we derive bounds on the sampling error. In this section, we utilize Lemma~\ref{lemma:precursor} to prove Lemma~\ref{lemma:sampling_error}. 

\subsection{Proof of Lemma \ref{lemma:sampling_error}}

\lemmathree*
\begin{proof}
\label{proof:data-term}
Using the same notation for normalized weights ($\theta_i$) as Appendix \ref{sec:proof1}, the difference between the virtual estimate $\bar{p}_c(y)$ and the actual estimate $\hat{p}_c(y)$ can be expressed as 
\begin{align}
    \label{sq:ster1}
    |\bar{p}_c(y) - \hat{p}_c(y)| = \left| \, \sum_{i=1}^n \theta_i (p_c(y) - \mathbbm 1_{c_i}(c)) \, \right|. 
\end{align}
The first term $ \sum_{i=1}^n \theta_i p_c(y)$ is the virtual estimate the classifier in \eqref{eq:nwc} would produce if it had true probability labels rather than discrete labels. Here, for one-hot vectorized class label representations $\mathbbm 1 _{c_i} \!\coloneqq\! \mathbbm 1 \{ c_i\! =\! c \}$,  $\mathbbm 1_{c_i}(c)$ represents the indicator function value at class index $c$. For convenience in notations, we denote the error between true and discrete labels $p_c(y) - \mathbbm 1_{c_i}(c)$ as $\varepsilon_i$.  

Since $c$ is a Bernoulli random variable with success probability $q_c$, the random variable $q_c - c$ is $\sigma$-sub-Gaussian with $\sigma \leq 1/4$ \cite[Thm~2.1 and Lemma~2.1]{buldygin2013subGaussiannormbinary}. Expanding this scalar result into the $\mathbb R^d$ problem setting, the context label $c$ is derived from the indicator function $\mathbbm 1_{c_i}(c)$ and the success probabilities $p_c(y)$. Therefore, the term $\varepsilon_i$ can be conceived of as an i.i.d.\ noise term that is $\sigma$-sub-Gaussian with $\sigma \leq 1/4$. With this observation, we can frame $\theta_i$ as a stochastic process since its values are always lesser than $1$, and $\varepsilon_i$ can be conceptualized as an i.i.d.\ $\sigma$-sub-Gaussian stochastic process. 

Consequently, the term $\sum_{i=1}^n \theta_i \varepsilon_i$ can be likened to $S_n$ from Lemma \ref{lemma:precursor} and $\sum_{i=1}^n \theta_i ^2$ to $V_n$. 

Expanding the $\theta_i$ notation out into its original form, we see that the term
\begin{align}
    \left| \sum_{i=1}^n \theta_i \varepsilon_i \right|  = \frac{1}{\kappa_n(y)} \left| \sum_{i=1}^n K_\lambda ^2 (y, y_i) \varepsilon_i  \right| \nonumber
\end{align}
is upper bounded with probability at least $1 - \delta$ by 
\begin{align}
    \label{eq:sterbound1}
    \frac{\sigma}{\kappa_n(y)} \sqrt{2 \log \left( \frac{1}{\delta}\sqrt{1 + \sum_{i=1}^n K_\lambda ^2 (y, y_i)} \right) \left( 1 + \sum_{i=1}^n K_\lambda ^2 (y, y_i)  \right)}.
\end{align}
Furthermore, since $K_\lambda (y, y_i) \leq 1$ (see  Assumption \ref{ass:kernel}), it follows that $K_\lambda ^2 (y, y_i) \leq K_\lambda (y, y_i)$. This allows us to upper bound the summation term as
\begin{align}
    \sum_{i = 1}^n K_\lambda ^2 (y, y_i) \leq \sum_{i=1}^n K_\lambda (y, y_i). \nonumber
\end{align}
The term on the right is, by definition, $\kappa_n(y)$. Substituting this into the bound, we obtain
\begin{align}
    \frac{1}{\kappa_n(y)} \left| \sum_{i=1}^n K_\lambda ^2 (y, y_i) \varepsilon_i  \right| \leq \sigma \sqrt{2 \log (\delta^{-1} \sqrt{1 + \kappa_n(y)}} \frac{\sqrt{1 + \kappa_n(y)}}{\kappa_n(y)}. \nonumber
\end{align}
Additionally, we can observe that if $\kappa_n(y) > 1$, then
\begin{align}
    \frac{\sqrt{1 + \kappa_n(y)}}{\kappa_n(y)} < \frac{\sqrt{2 \kappa_n(y)}}{\kappa_n(y)} = \frac{\sqrt 2}{\sqrt{\kappa_n(y)}}.  \nonumber
\end{align}
Therefore, with probability at least $1- \delta$, for $\kappa_n(y) > 1$, 
\begin{align}
    \label{eq:sterfinalone}
    \frac{1}{\kappa_n(y)} \left| \sum_{i=1}^n K_\lambda ^2 (y, y_i) \varepsilon_i  \right| \leq \frac{2\sigma}{\kappa_n(y)} \sqrt{\kappa_n(y) \log(\delta^{-1}\sqrt{1 + \kappa_n(y)})},
\end{align}
and for $0 < \kappa_n(y) \leq 1$, 
\begin{align}
    \label{eq:sterfinaltwo}
    \frac{1}{\kappa_n(y)} \left| \sum_{i=1}^n K_\lambda ^2 (y, y_i) \varepsilon_i  \right| & \leq  \sigma \sqrt{2 \log \left(\delta^{-1} \sqrt{1 + \kappa_n(y)}\right)} \frac{\sqrt{1 + \kappa_n(y)}}{\kappa_n(y)}.
\end{align}
We can see here that if $\kappa_n(y) \leq 1$, then the term $\sqrt{1 + \kappa_n(y)} \leq \sqrt 2$. By taking the worst-case bounds, we can thus simplify the term in \eqref{eq:sterfinaltwo} to 
\begin{align}
    \label{eq:sterfinalthree}
    \sigma \sqrt{2 \log \left(\delta^{-1} \sqrt{1 + \kappa_n(y)}\right)} \frac{\sqrt{1 + \kappa_n(y)}}{\kappa_n(y)} \leq  \frac{2 \sigma}{\kappa_n(y)} \sqrt{\log \left( \frac{\sqrt 2}{\delta} \right)}.
\end{align}
By collecting the terms from \eqref{eq:sterfinalone} and \eqref{eq:sterfinalthree}, we arrive at 
\begin{align}
    \label{eq:sterfinalfinal}
    |\bar{p}_c(y) - \hat{p}_c(y)| \leq 
    \begin{cases}
        \frac{2\sigma}{\kappa_n(y)} \sqrt{\kappa_n(y) \log(\delta^{-1}\sqrt{1 + \kappa_n(y)})} & \qquad \text{if } \kappa_n(y) > 1, \\ 
        \frac{2 \sigma}{\kappa_n(y)} \sqrt{\log \left( \frac{\sqrt 2}{\delta} \right)} & \qquad \text{if } 0 < \kappa_n(y) \leq 1.
    \end{cases}
\end{align}
\end{proof}

\section{LIPSCHITZ-CONTINUOUS SYNTHETIC DATASET GENERATION}
\label{sec:synthetic_data_creation}

In this section, we provide details on how we generated synthetic data (see Section \ref{sec:synthtic_data_para}) to match Assumption~\ref{ass:lipschitz}. We define an underlying probability function $p_c(y, L)$ that is a function of a known Lipschitz constant $L$. Instead of generating hard-labelled measurements, we sample data points according to this function. This is easy to achieve in a case of binary classification, where the relationship between probabilities is exact. That is, for two classes, we have a single separating hyperplane. The probability for a sample $y$ can be modelled using the logistic (sigmoid) function, which takes the signed distance to the hyperplane as input. The probability  $p_c$ for class $c=1$ would be
\begin{align}
    \label{eq:logistic}
    p_c(y) = \frac{1}{1 + \exp(-k (wy + b))}, 
\end{align}
where $w$ is the normal vector to the hyperplane with $\| w \| = 1$ and $k$ is a scaling factor that controls the steepness of the probability function. The gradient of this function, 
\begin{align}
    \nabla p_1(y) = k \cdot p_1(y) \cdot (1 - p_1(y)) \cdot w, 
 \end{align}
is maximized at $p_1(y) = 0.5$, which leads to a direct relationship that $L = k/4$. 

\section{ESTIMATING A LIPSCHITZ CONSTANT FROM DATA}
\label{sec:estimating-L}

The Lipschitz constant $L$ in Assumption \ref{ass:lipschitz} and the equivalent margin in Assumption~\ref{ass:separable} are parameters that convey the underlying assumptions about the smoothness of the data. It influences the strength of the classifier's bounds. However, since the underlying probability function $p_c(y)$ cannot be measured, we must estimate this constant from data. Correspondingly, it is also important to first determine if the dataset falls into the category of overlapping or separable distributions. 

To approximate the Lipschitz constant from data, the maximum pairwise distance forms a basis for an upper bound. Here, we assume that for two samples $y$ and $y'$ with matching labels that are the closest to each other in the entire dataset, the true probability $p_c(y)$ and $p_c(y')$ do not differ by more than a threshold probability $P_t$ for any $c \in \mathcal C$. Then, we can estimate the Lipschitz constant from \eqref{eq:lipschitz} as 
\begin{align}
    \label{eq:l-estimation}
    L \geq \frac{P_t}{\sup \{ \| y - y' \|  \}}. 
\end{align}
Lastly, to determine the nature of the data distribution, we selected \SI{1000} instances from the MNIST and MIT-BIH dataset with random stratified sampling. We computed the pairwise distances from each sample to all other samples in this set. We observed that the maximum pairwise distance within a class was comparable to the maximum pairwise distance observed globally. For both datasets, this implies that their distributions are overlapping, rather than being separated by a margin. Additionally, in Figure~\ref{fig:kernels_hyperopt}, we show $K(v)$ for all pairs of a random selection of \SI{100} samples from the MNIST dataset for its optimized bandwidth $\lambda=7.5$. 

In Remark~\ref{remark:lipschitz} (Section~\ref{sec:problem}), we introduce various approaches in the literature that have been used to estimate Lipschitz constants from measurements. Most methods are heuristic, including the one we use based on \cite{Strongin1973}. Estimating an upper bound on the smoothness of a function with guarantees is intractable without invoking further regularity assumptions. This is done so in the form of assuming bounded higher-order derivatives \citep{huangSampleComplexityLipschitz}, or by invoking sampling assumptions about the measurements \cite{tokmak2025safeexplorationreproducingkernel}. The former is best suited for approaches where it is more reliable to know a limit on the second-order derivative: for example, a motor is bounded by its maximum torque. On the other hand, the latter approach is better in cases where we can independently sample from a family of functions in a probability space: for example, an experimentally determined noise oracle. In our case, $p_c(y)$ is entirely unknown; we neither have knowledge of a bounded higher-order derivative nor a reliable class of functions from a probability space to sample from. Therefore, in this paper, we approximate the Lipschitz constant from available data. 

\section{SUPPLEMENTARY RESULTS}
\label{sec:supplementary_results}

In this section, we report further details from the experiments conducted in Section~\ref{sec:experiments} and also provide further experimental results.

\subsection{Additional figures for dyadic classifier predictions}
\label{sec:dyadic_grids}

In the \textit{dyadic} approach, we construct a hash table that maps the grid index to the aggregated class count vector. Figure~\ref{fig:dyadic_cells} shows two synthetic datasets corresponding to Assumptions~\ref{ass:lipschitz} and~\ref{ass:separable}, and their respective hash tables marked with majority predictions. 

\begin{figure}[h]
    \centering
    \begin{minipage}[t]{\textwidth}
    \centering
    \includegraphics{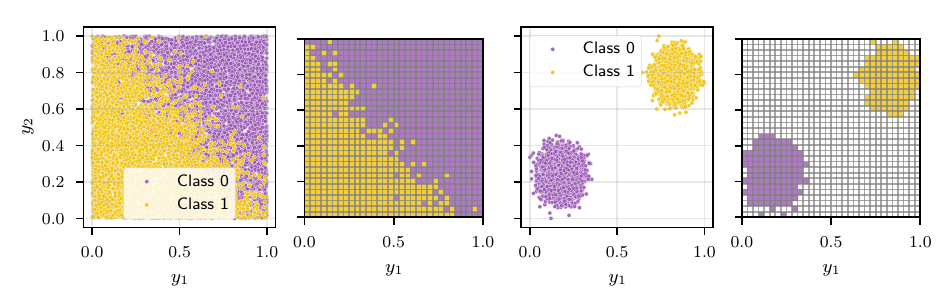}
    \end{minipage}
    \caption{Two synthetic datasets and their corresponding dyadic cell prediction grids. \capt{The figures on the left correspond to the Lipschitz-continuous overlapping dataset; the figures on the right correspond to the dataset separated by a margin.}}
    \label{fig:dyadic_cells}
\end{figure}

\subsection{Additional results from the ECG dataset}
\label{sec:additonal_ecg}

In this section, we show additional results from the experiments conducted in Section~\ref{sec:experiments} on the ECG dataset. Figure~\ref{fig:ecg_waveforms} shows randomly selected ECG waveforms from the testing set and their associated class prediction probabilities. In Figure~\ref{fig:ece} we show the alignment between prediction confidence ($\max \hat p_c$) and accuracy. This can be used to compute the ECE \citep{guo2017CalibrationModernNeural}, which in our case is \SI{7.5}{\percent}. 




\begin{figure*}[h!]
    \centering
    \begin{minipage}{\textwidth}
        \centering
        \resizebox{0.9\textwidth}{!}{\input{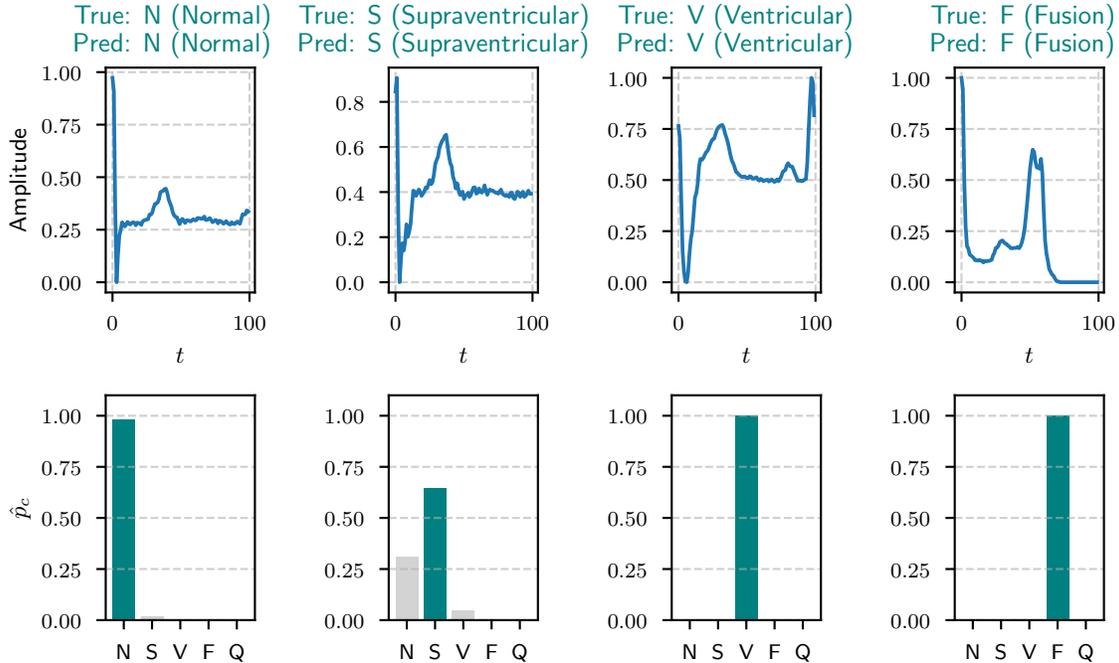}}
    \end{minipage}
    \caption{Illustrative ECG waveforms (amplitude normalized from \si{mV}, $t$ denotes time step) and associated class probabilities for each class from the MIT-BIH database; class Q (unclassifiable) has been excluded.}
    \label{fig:ecg_waveforms}
\end{figure*}



Next, following the conventions of \cite{nnyaba2024EnhancingElectrocardiographyData}, we define a Type I error as the instance where the classifier's prediction, i.e., $\max \hat{p}_c$, is incorrect, and a Type II error as an instance where the intervals are so wide that $\hat{p}_c$ could be less than $0.5$. In Figure~\ref{fig:type} we show the Type I and Type II error counts obtained from the \textit{regular} classifier's predictions on the entire testing set. Here, we also see that the errors are overrepresented in the minority classes. If we count Type II errors as incorrect predictions, the accuracy of our classifier is reduced to \SI{84}{\percent}. 

\begin{figure*}[h]
    \centering
    
    \begin{minipage}[t]{0.47\textwidth}
        \centering
        \resizebox{\textwidth}{!}{\input{figures/ece.pgf}}
        \caption{Accuracy of the \textit{regular} NWC grouped by $\max \hat{p}_c$ bins. \capt{The figure shows how the classifier's prediction confidence correlate with its accuracy.}} 
        \label{fig:ece}
    \end{minipage}
    \hfill
    \begin{minipage}[t]{0.47\textwidth}
        \centering
        \resizebox{\textwidth}{!}{\input{figures/type_errors_by_class.pgf}}
        \caption{Type I and Type II error counts for each class for the \textit{regular} classifier's predictions over the entire testing set. \capt{The inset text also shows the percentage of the dataset these counts represent.}} 
        \label{fig:type}
    \end{minipage}
\end{figure*}
\subsection{Results on the MNIST dataset}
\label{sec:mnist}

In addition to synthetic and ECG datasets, we implemented the proposed classifier on the MNIST handwritten digits dataset \citep{deng2012mnist}. The dataset contains $28 \times 28$ pixel black-and-white images of handwritten digits commonly used for training various image processing classifiers. The dataset contains \SI{60000} training samples split equally among classes and \SI{10000} training samples. Figure~\ref{fig:mnist_probs} shows predictions, accuracy trends, and bound trends for the \textit{regular} classifier trained on the MNIST dataset with $L = 0.03$ and $\lambda = 7.5$. From the figure, we observe that the classifier demonstrates an accuracy of $>$\SI{92}{\percent} on the dataset with strong bounds with just \SI{10000} samples. 

\begin{figure}[h!]
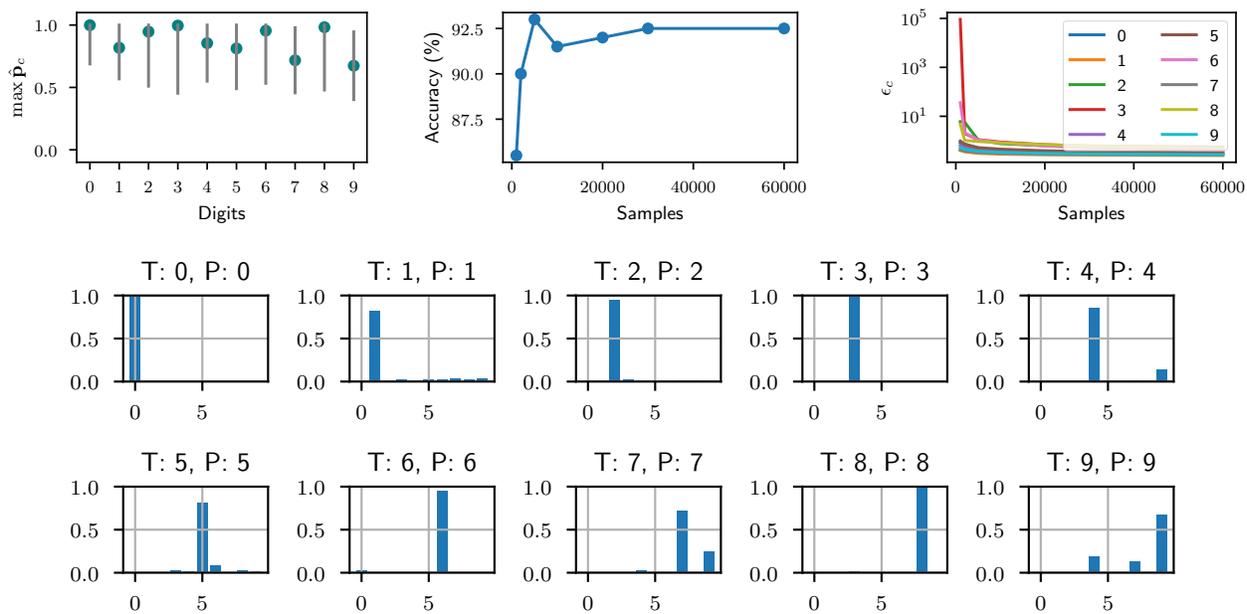

    \centering
    \begin{minipage}[t]{0.32\textwidth}
        \centering
        \resizebox{\textwidth}{!}{\input{figures/predictions.pgf}}
        \label{fig:mnist_predictions}
    \end{minipage}%
    \hfill
    \begin{minipage}[t]{0.32\textwidth}
        \centering
        \resizebox{\textwidth}{!}{\input{figures/accuracy.pgf}}
        \label{fig:mnist_accuracy}
    \end{minipage}%
    \hfill
    \begin{minipage}[t]{0.32\textwidth}
        \resizebox{\textwidth}{!}{\input{figures/bounds_tightening.pgf}}
        \label{fig:mnist_bounds}
    \end{minipage}
    \begin{minipage}{\textwidth}
        \vspace{-2em}
        \centering
        \begin{adjustbox}{clip,trim=0cm 2cm 0cm 2cm}
            \resizebox{0.9\textwidth}{!}{\input{figures/number_probs.pgf}}
        \end{adjustbox}
        \caption{Results of the regular classifier on the MNIST dataset. \capt{On the top row, we see predictions (left) and tightness of the bounds for 10 randomly selected samples of each digit from the dataset. Their corresponding all-class prediction probabilities are shown in the figure below, where true and predicted labels are marked with `T' and `P' respectively. Next on the top row, we show the accuracy (middle) and strength of the bounds (right) of the classifier with varying number of samples $n$.}}
        \label{fig:mnist_probs}
    \end{minipage}
\end{figure}

\subsection{Results on a simplified GTSRB dataset}
\label{sec:gtsrb}

\begin{figure}[h!]
    \centering
    \begin{minipage}[t]{\textwidth}
    \centering
    \includegraphics[width=0.9\textwidth]{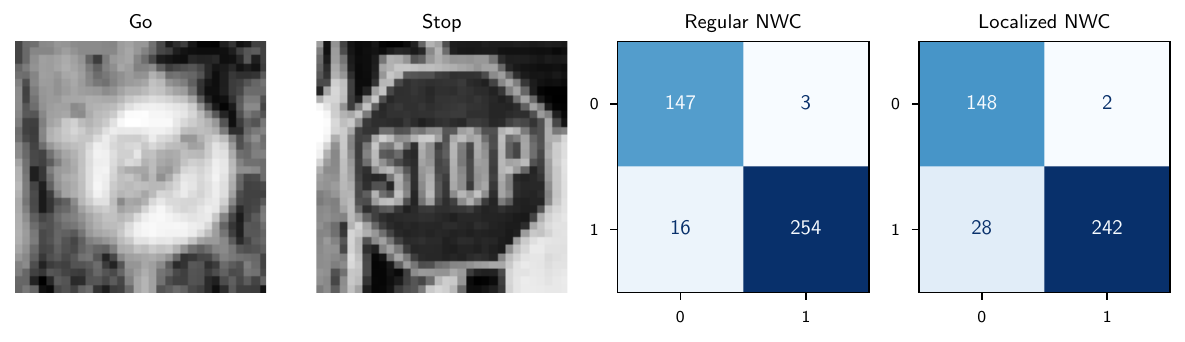}
    \end{minipage}
    \caption{Results from the reduced GTSRB dataset. \capt{On the left are the two classes considered in our study. On the right are confusion matrices from the two proposed variants of our classifier. In the confusion matrices, the label 0 corresponds to the `go' sign and 1 corresponds to the `stop' sign.}}
    \label{fig:gtsrb}
\end{figure}

The improved computational efficiency of the Nadaraya-Watson estimator is of significant advantage when applied to higher-dimensional systems. While we overcome the computational bottleneck associated with GP-like methods, we are still limited by how expressive distances can be in higher dimensional spaces. These limits are best explored and tested in the context of high-dimensional RGB images. We evaluate our \textit{regular} and \textit{localized} classifier on a subset of the German Traffic Sign Recognition Benchmark (GTSRB) dataset. We simplify to a binary classification problem including only the `stop' sign and `end of speed limit' sign. Figure~\ref{fig:gtsrb} shows these results with $L = 0.05$, $\lambda = 7.5$, and $k=20$ for the \textit{localized} classifier. 

\subsection{Comparison with other baselines}
\label{sec:baselines}

Tasked with ECG heartbeat classification, various classical and deep learning methods have achieved highly accurate results on the MIT-BIH database. Our contribution lies not only in achieving high accuracy scores but also in providing actionable uncertainty quantification in relation to the classifier's predictions. These theoretical guarantees are often entirely absent in such methods. Nevertheless, in this section, we summarize and review the accuracy obtained by such methods (those cited in Section~\ref{sec:discussion-related-work}) to contextualize our work's performance with respect to existing methods. For a more extensive comparison, we refer the reader to Table 1 and Table 2 from~\cite{sattar2024}. 
\begin{table}[ht]
    \centering
    \caption{Summary of reported ECG classification results in the literature.}
    \begin{tabularx}{\linewidth}{
        p{0.24\linewidth} 
        X                 
        >{\centering\arraybackslash}p{0.12\linewidth}
        >{\centering\arraybackslash}p{0.1\linewidth}
    }
        \toprule
        \textbf{Literature} & \textbf{Remark} & \textbf{Accuracy} & \textbf{Bounds} \\
        \midrule
        \cite{kumari2022ClassificationECGbeatsa} & Optimized decision tree & \SI{98.77}{\percent} & \ding{55} \\
        \cite{zhou2024MultimodalECGheartbeat} & CNN with frequency channel attention &  \SI{99.6}{\percent} & \ding{55} \\
        \cite{abdalla2020Deepconvolutionalneural} & CNN with 11 layers & \SI{99.84}{\percent} & \ding{55} \\
        \cite{gao2019EffectiveLSTMRecurrent} & Long short-term memory model with focal loss & \SI{99.26}{\percent} & \ding{55} \\
        \cite{jha2020Cardiacarrhythmiaclassification} & SVM with tunable Q-wavelet transform & \SI{99.27}{\percent} & \ding{55} \\
        \cite{nnyaba2024EnhancingElectrocardiographyData} & GP classification with bounds & $\approx$ \SI{98}{\percent} & \checkmark \\
        \bottomrule
    \end{tabularx}
\end{table}

\section{ABLATIONS}
\label{sec:ablations}

In this section, we perform ablation studies on kernel choice and bandwidth parameter $\lambda$. As the Nadaraya-Watson estimator is a non-parametric kernel regressor, we study the effect different kernel functions have on the nature of the estimate. Table~\ref{tab:kernels} lists the different kernel functions we evaluate our classifier with. Correspondingly, for each kernel, we optimize to find the bandwidth $\lambda$ that maximizes the classifier's accuracy and bounds.  

To this end, we define a weighted objective function that tries to balance accuracy $A$ and average uncertainty interval $B$ as a function of a user-defined weight $r$: 
\begin{align}
    \label{eq:hyperopt}
    J(A, B) = r A - (1 - r )B.
\end{align}
Here, we utilize a simple weighted-sum that awards a higher score if $A$ is maximized and $B$ is minimized. The weight $r$ directly quantifies the trade-off between the two potentially competing measures. 

\begin{table}[htpb]
    \caption{Different kernel functions and their definitions.}
    \vspace{1em}
    \centering
    \begin{tabular}{l c}
        \toprule
        \textbf{Kernel} & \textbf{Definition} \\
        \midrule
        Boxcar kernel & $\displaystyle K(v) = 
                        \begin{cases} 
                        1 & \text{if } \|v\| \leq 1 \\ 
                        0 & \text{otherwise} 
                        \end{cases}$ \\
        Gaussian kernel (without truncation) & $\displaystyle K(v) = \exp\left(-\frac{\|v\|^2}{2}\right)$ \\
        Epanechnikov kernel & $\displaystyle K(v) = 
                \begin{cases} 
                1 - v^2 & \text{if } \|v\| \le 1 \\ 
                0 & \text{otherwise} 
                \end{cases}$ \\ 
        Quartic kernel & $\displaystyle K(v) = 
                \begin{cases} 
                (1 - v^2)^2 & \text{if } \|v\| \le 1 \\ 
                0 & \text{otherwise} 
                \end{cases}$ \\
        Triweight kernel & $\displaystyle K(v) = 
                \begin{cases} 
                (1 - v^2)^3 & \text{if } \|v\| \le 1 \\ 
                0 & \text{otherwise} 
                \end{cases}$ \\
        Tricube kernel & $\displaystyle K(v) = 
                \begin{cases} 
                (1 - |v|^3)^3 & \text{if } \|v\| < 1 \\ 
                0 & \text{otherwise} 
                \end{cases}$ \\
        Cosine kernel & $\displaystyle K(v) = 
                        \begin{cases} 
                        \frac{\pi}{4}\cos\left(\frac{\pi}{2}v\right) & \text{if } \|v\| \le 1 \\ 
                        0 & \text{otherwise} 
                        \end{cases}$ \\
        \bottomrule
    \end{tabular}
    \label{tab:kernels}
\end{table}

The optimization experiments were conducted on the MNIST dataset. We trained the \textit{regular} NWC on \SI{30000} samples of the MNIST dataset with all the kernels from Table~\ref{tab:kernels}. We set the weight $r$ in \eqref{eq:hyperopt} as $0.95$. Since $A$ and $B$ are experimentally computed, there is no direct way to find derivatives of $J$ with respect to $\lambda$; thus, we used Powell's conjugate direction method to optimize for the score $J$ \citep{powell}. The results from Figure~\ref{fig:kernels_hyperopt} show that the Epanechnikov kernel offered the best performance. Yet, as we see from the $y$-axis of the figure, the improvement difference was marginal. This aligns with previous research that have shown the limited effect of kernel choice on the Nadaraya-Watson estimator \citep{hardle2002Appliednonparametricregression}.

\begin{figure}[htpb]
    \centering
    \begin{minipage}[t]{0.48\textwidth}
        \centering
        \resizebox{\textwidth}{!}{\input{figures/kernels.pgf}}
    \end{minipage}
    \begin{minipage}[t]{0.48\textwidth}        
        \centering
        \includegraphics[width=0.92\textwidth]{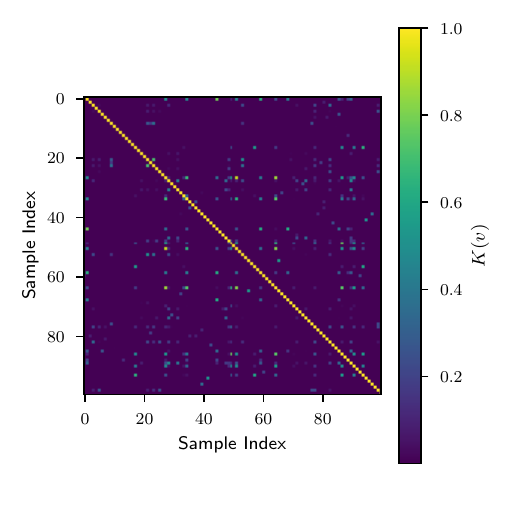}
    \end{minipage}
    \caption{Hyperparameter optimization on the kernels. \capt{On the left, we show the baseline and optimized score $J(A, B)$ on all kernels from Table~\ref{tab:kernels}. Here, the results show that the effect of hyperparameter optimization to find the optimal $\lambda$ on each kernel is, at best, only moderately significant. On the right, we show the Epanechnikov kernel evaluated at all pairs of 100 randomly sampled MNIST images with optimized $\lambda = 7.5$.}}
    \label{fig:kernels_hyperopt}
\end{figure}
\end{document}